\documentclass[conference]{IEEEtran}
\IEEEoverridecommandlockouts
\usepackage{amsfonts}
\usepackage{cite,graphicx,amsmath,amsthm}
\allowdisplaybreaks[4]
\usepackage{subcaption}
\usepackage{fancyhdr}
\usepackage{dsfont}
\usepackage{array,color,xcolor}
\usepackage{bm}
\usepackage{booktabs}
\usepackage{multirow}
\usepackage{algorithm}
\usepackage{algpseudocode}

\usepackage{bbm}
\usepackage{graphicx}
\usepackage{amsmath}

\newtheorem{theorem}{Theorem}

\newtheorem{lemma}{Lemma}

\newtheorem{proposition}{Proposition}

\newtheorem{corollary}{Corollary}

\newtheorem{property}{Property}

\newtheorem{remark}{Remark}

\newtheorem{claim}{Claim}

\usepackage[a4paper, top=1.82cm, bottom=4.4cm, left=1.4cm, right=1.4cm]{geometry}

\allowdisplaybreaks[4]

\definecolor{purple2}{HTML}{EEE6FD}
\usepackage[table]{xcolor}
\usepackage{flushend}

\begin{document}

\title{\huge Memory Centric Power Allocation for Multi-Agent\\Embodied Question Answering
}

\author{
    \IEEEauthorblockN{
    Chengyang Li\IEEEauthorrefmark{1}, 
    Shuai Wang\IEEEauthorrefmark{2}, 
    Kejiang Ye\IEEEauthorrefmark{2}, 
    Weijie Yuan\IEEEauthorrefmark{3},
    Boyu Zhou\IEEEauthorrefmark{3}, \\
    Yik-Chung Wu\IEEEauthorrefmark{1},
    Chengzhong Xu\IEEEauthorrefmark{4}, \emph{Fellow, IEEE}
    , and Huseyin Arslan\IEEEauthorrefmark{5}, \emph{Fellow, IEEE}
    }
    \vspace{-0.2in}
    \thanks{\scriptsize   
    This work was supported by the National Key R\&D Program of China (No. 2025YFE0204100), Science and Technology Development Fund of Macao S.A.R (FDCT) under number 0074/2025/AMJ, the CAS-TUBITAK Joint Call under the International Partnership Program of the Chinese Academy of Sciences (Grant No. 321GJHZ2025118M) and The Scientific and Technological Research Council of Türkiye (Grant No. 225N080), and the Shenzhen Science and Technology Program (Grant No. RCYX20231211090206005). 

    \IEEEauthorrefmark{1} Chengyang Li and Yik-Chung Wu are with Department of Electrical and Computer Engineering, The University of Hong Kong, Hong Kong.
    
    \IEEEauthorrefmark{2} Shuai Wang and Kejiang Ye are with Shenzhen Institutes of Advanced Technology, Chinese Academy of Sciences, Shenzhen, China. 
    
    \IEEEauthorrefmark{3} Weijie Yuan and Boyu Zhou are with Southern University of Science and Technology, Shenzhen, China.
    
    \IEEEauthorrefmark{4} Chengzhong Xu is with University of Macau, Macau SAR, China.
    
    \IEEEauthorrefmark{5} Huseyin Arslan is with Istanbul Medipol University, Istanbul, Turkey.
    
    Corresponding author: Shuai Wang (s.wang@siat.ac.cn).}
}

\maketitle

\begin{abstract}
This paper considers multi-agent embodied question answering (MA-EQA), which enables robot teams to answer queries based on their long-horizon observations. In contrast to existing edge resource management methods that optimize sensing, communication, or computation performance metrics, MA-EQA focuses on the quality of aggregated memory. To address this paradigm shift, we propose a quality of memory (QoM) model based on generative adversarial exam (GAE), which leverages forward simulation to evaluate memory retrieval and utilizes the resulting exam scores to quantify QoM. Based on the QoM model, we develop a memory-centric power allocation (MCPA) scheme that maximizes memory quality under communication resource constraints. Through analytical characterization in the noise-limited regime, we reveal a GAE-augmented capped water-filling structure for MCPA. Extensive experiments demonstrate that MCPA achieves significant improvements over existing benchmarks across diverse metrics and scenarios. 
\end{abstract}

\begin{IEEEkeywords}
Embodied question answering, multi-agent systems, power allocation, quality of memory.
\end{IEEEkeywords}

\section{Introduction}

Embodied question answering (EQA) enables users to query the robots about what they have observed in their memories \cite{sermanet2024robovqa}.
Standalone EQA solutions are constrained by limited memory capacity \cite{anwar2025remembr}. 
To address this limitation, multi-agent EQA (MA-EQA) aggregates the local memory of distributed robots over wireless networks, building a global memory with increased spatial and temporal coverage \cite{li2025embodied,wang2026towards}.
However, in contrast to existing methods that adopt sensing coverage \cite{cheng2025development}, communication throughput \cite{ye2025integrated}, computation efficiency \cite{wang2026low}, or their combinations \cite{jiang2025integrated}, as design objectives, MA-EQA aims to maximize the memory quality.
Therefore, these schemes \cite{cheng2025development,ye2025integrated, wang2026low, jiang2025integrated} may overlook the non-uniform memory contributions of different agents, resulting in degraded EQA accuracy.

To prioritize transmissions based on data value, semantic communication (SemCom) \cite{liu2025intelligent,yan2022resource} reduces information redundancy while preserving semantic similarity (e.g., cosine similarity between embeddings). However, semantic similarity is an intermediate rather than end-to-end task metric for EQA, which may overlook subtle but critical differences in long memory contexts.
On the other hand, task-oriented communication \cite{wang2020machine,shi2023task,wen2023task,11143883} enables end-to-end optimization of inference performance within integrated sensing, communication, and computation systems. 
If the tasks are closed-set with well-defined object classes and labels \cite{wang2020machine,wen2023task,shi2023task}, the objective function can be explicitly formulated. 
In MA-EQA, however, the task is open-set, and the memory must be constructed without prior knowledge of the questions \cite{anwar2025remembr}, which makes it difficult to find a proper memory-oriented objective function.

To fill the gap, we introduce memory centric power allocation (MCPA) that unifies memorization and communication into an end-to-end optimization framework.
Our solution is inspired by a key observation: ``While the specific questions users may ask are infinite and unpredictable, the \emph{patterns and templates of these questions are finite and predictable}.'' 
As such, we build a generative adversarial exam (GAE) system with proxy questioner, enabling the \emph{forward simulation of potential EQA processes}. 
This enables us to transform the implicit EQA accuracy into the explicit quality of memory (QoM), which is defined as a composite function of GAE scores and communication rates.
By maximizing the QoM under communication resource constraints, MCPA leads to a paradigm shift from sensing, communication, or computation centric towards \emph{memorization centric optimization}.
To get deeper insights into MCPA, an analytical solution is derived for the noise-limited case. The solution reveals a GAE-augmented capped water-filling structure for MCPA. Extensive results in the high-fidelity CARLA simulation show that MCPA achieves significant improvements over existing schemes in terms of both QoM and EQA accuracy. 
By analyzing the communication rate and sensing coverage performances, it is found that communication and sensing functions are not always beneficial to memorization. Their positive contribution is conditioned on the associated QoMs.

\section{System Model}\label{section2}

We consider the MA-EQA system shown in Fig.~\ref{fig:fig1}, which consists of an edge server and $K$ robots (e.g., drones).
The system builds a long-horizon global memory $\mathcal{M}$ by aggregating distributed data $\mathcal{D}=\{\mathcal{D}_1,\cdots,\mathcal{D}_K\}$ from robots $\mathcal{K}=\{1,\cdots,K\}$.
Based on $\mathcal{M}$, MA-EQA outputs answers $\mathcal{A}=\{\mathbf{a}_1, \mathbf{a}_2,\cdots\}$ to questions $\mathcal{Q}=\{\mathbf{q}_1, \mathbf{q}_2,\cdots\}$, where $\mathbf{q}_j$ and $\mathbf{a}_j$ are text vectors of the $j$-th question and $j$-th answer, respectively. 
The design variables that can be controlled are the transmit powers of robots $\mathbf{p}=[p_{1},\cdots,p_{K}]^T$, whose feasible set is $\mathcal{P}=\{\mathbf{p}\succeq \mathbf{0}: \sum_{k=1}^{K}p_{k} \leq P_{\mathrm{sum}}\}$, with $P_{\mathrm{sum}}$ being the total power budget to ensure low power consumption and interference leakage.

\begin{figure}[!t]
    \centering
    \includegraphics[width=0.47\textwidth]{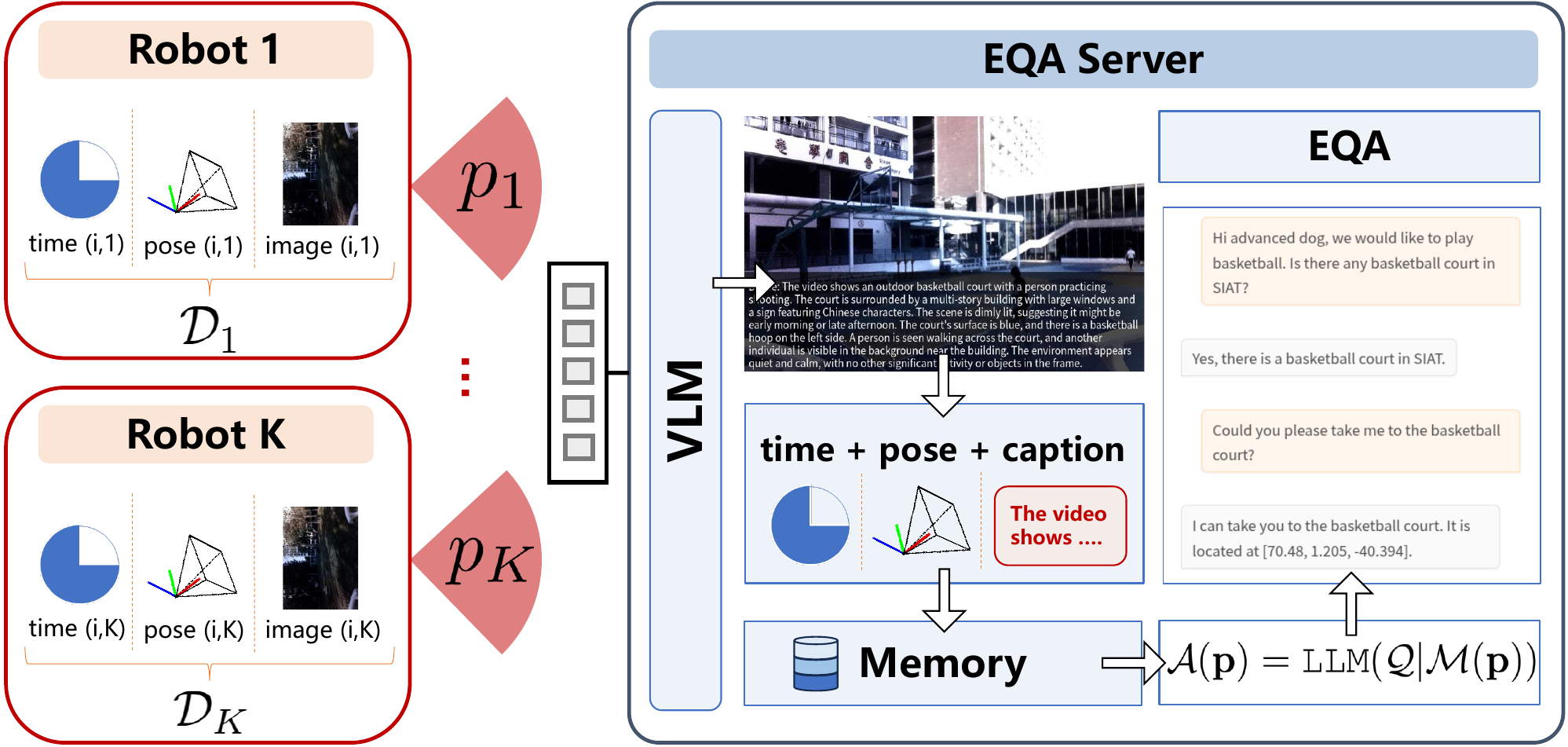}
    \caption{The MA-EQA system model.}
    \label{fig:fig1}
        \vspace{-0.15in}
\end{figure}

Each memory item is defined as a tuple of (time, pose, image) data $\mathcal{I}_{i,k}=(t_{i,k},\mathbf{s}_{i,k},\mathbf{v}_{i,k})$, where $t_{i,k}$ is the timestamp, $\mathbf{s}_{i,k}\in \mathbb{R}^{6}$ is the robot pose, $\mathbf{v}_{i,k}\in \mathbb{R}^{LWC}$ is the camera image, with $(L,W,C)$ denoting the length, width, and channel numbers for each image \cite{anwar2025remembr}.
Each robot possesses a set of items $\mathcal{D}_k=\left\{\mathcal{I}_{i,k}\right\}_{i=1}^{|\mathcal{D}_k|}$, where the data volume of each $\mathcal{I}_{i,k}$ is $Z_k$.
The multi-array channel from robots to server is $\mathbf{h}_k\in\mathbb{C}^N$, where $N$ is the number of antennas at server.
With this model, the channel gains under maximum ratio combining are given by $H_k= \|\mathbf{h}_k\|_2^2$ and 
$I_{k,j}=|\mathbf{h}_k^H\mathbf{h}_{j}|^2\left\Vert\mathbf{h}_{k}\right\Vert_2^{-2}$, 
where term $H_{k}=I_{k,k}$ represents the channel gain of robot $k$, while the term $I_{k,j}$ represents the interference channel between robots $j$ and $k$ \cite{wang2020machine}.
The number of frames uploaded from robot $k$ is
\begin{align}
    F_k(\mathbf{p})=\min\left[|\mathcal{D}_k|, \frac{TB}{Z_{k}}\mathrm{log}_2\left(1+\frac{H_{k}p_{k}}{\sum_{j\neq k}^KI_{k,j}p_{j}+
\sigma^2} \right)\right], \nonumber
\end{align}
where $T$ is the time budget, $B$ is the bandwidth, $\sigma^2$ is the power of complex Gaussian noise \cite{wang2020machine}.
The aggregated memory data is thus
\begin{align}
&\mathcal{E}(\mathbf{p}) =\bigcup_{1\leq k\leq K}
\bigcup_{i\in\mathcal{S}_k(\mathbf{p})}
\Big\{t_{i,k},\mathbf{s}_{i,k},\mathbf{v}_{i,k}\Big\},
\label{data}
\end{align}
where $\mathcal{S}_k(\mathbf{p})=\{
\mathcal{I}_{i_1,k},\cdots,\mathcal{I}_{i_{F_k(\mathbf{p})},k}
\}\subseteq\mathcal{D}_k$.

Next, the server captions images in $\mathcal{E}$ using vision-language model \texttt{VLM} as
$\mathbf{c}_{i,k} = \texttt{VLM}\left(\mathbf{v}_{i,k}\right), \ \forall x_{i,k}=1$,
which generates a set of captions $\mathcal{C} = \{\mathbf{c}_{i,k}\}$. 
We consider compressing $\mathcal{C}$ using a text embedding function $E$, and uses a vector database (e.g., \textit{Milvus}) to store (time, pose, embedding) pairs:
    \begin{align}\label{memory}
        \mathcal{M}(\mathbf{p})=&
        \mathcal{M}_0\bigcup 
        \bigcup_{1\leq k\leq K}
        \mathcal{M}_k,
        \\
        \mathcal{M}_k=&\bigcup_{i\in\mathcal{S}_k(\mathbf{p})}
        \left\{t_{i,k},\mathbf{s}_{i,k},E\Big[\texttt{VLM}\left(\mathbf{v}_{i,k}\right)\right]\Big\},
    \end{align}
where $\mathcal{M}_0$ is the historical database at the server before memory collection.
The \emph{vector memory} $\mathcal{M}$ is extremely efficient in search, and can be used to store millions of vector embeddings \cite{anwar2025remembr}.     
Conditioned on the memory $\mathcal{M}$ and questions $\mathcal{Q}$, the system aims to provide a set of answers 
$\mathcal{A}=\{\mathbf{a}_j\}$ such that 
\begin{align}\label{llm}
\mathcal{A}(\mathbf{p})=\texttt{LLM}(\mathcal{Q}|\mathcal{M}(\mathbf{p})),
\end{align}
where each $\mathbf{a}_j(\mathbf{p})=\texttt{LLM}(\mathbf{q}_j|\mathcal{M}(\mathbf{p}))$.

The goal of MA-EQA is to maximize the answer accuracy of $\mathcal{A}$ conditioned on $\mathcal{M}$.
Accordingly, we define the probability function as 
\begin{align}
    \Psi\left[\mathcal{M}(\mathbf{p})\right]
    :=
    \mathrm{Pr}\left(\mathcal{A}=\mathcal{A}^*|\mathcal{M}(\mathbf{p})\right),
\end{align}
where $\mathcal{A}^*$ is the ground truth answer set. 
Then we obtain the following MCPA problem:
\begin{align}
\mathsf{P}:~~\mathop{\mathrm{max}}_{\substack{\mathbf{p}\in\mathcal{P}}}
\quad & \Psi\left[\mathcal{M}(\mathbf{p})\right].
\end{align}
The function $\Psi$ maps a memory chunk $\mathcal{M}$ into a score between $0$ and $1$, i.e., $0\leq \Psi(\mathcal{M}) \leq 1$.
This enables us to distinguish the heterogeneous values of different memories. 
However, since $\mathcal{M}$ must be collected without knowing the question $\mathcal{Q}$ in advance, there is no access to $\Psi$ at the memory collection stage. 
This turns out to be the key challenge to solve $\mathsf{P}$.

\section{Quality of Memory Model}

In this section, we propose a novel QoM model to approximate $\Psi$. 
The key insight is that the potential question cannot be an arbitrary question.
The patterns and templates of questions are known in advance.
More importantly, to generate answers $\mathcal{A}$ for questions $\mathcal{Q}$, at least one robot must have seen the associated scene. Otherwise the question should be considered as invalid.
Based on the above insight, we can generate in-context questions $\mathcal{Q}_k$ from memory $\mathcal{M}_k$, and use $\mathcal{Q}_k$ to test the current memory $\mathcal{M}_0$. 
Because $\mathcal{Q}_k$ is at least answerable by $\mathcal{M}_k$, $\mathcal{Q}_k$ is guaranteed to be a valid question. 
Furthermore, if there exists any question in $\mathcal{Q}_k$ that cannot be answered by $\mathcal{M}_0$, then merging $\mathcal{M}_k$ is guaranteed to increase the QA capability of the system. 
In contrast, if $\mathcal{M}_0$ can answer all the questions in $\mathcal{Q}_k$, $\mathcal{M}_k$ would contribute to no new memory. 
This gives GAE shown in Fig.~\ref{fig:fig2} for computing QoM. 
Below we present the details of this framework.

\subsection{Generative Adversarial Exam}

The input of GAE consists of the distributed datasets $\{\mathcal{D}_k\}$ at robots and the pre-collection memory $\mathcal{M}_0$ at the server.
The output of GAE is the scores $\{\texttt{GAE}_k\in[0,1]\}$ for all $k$. The GAE pipeline is divided into the following steps.
\begin{itemize}
    \item[1)] Pilot upload: Each robot samples a sub-dataset $\{\widetilde{\mathcal{D}}_k \subseteq \mathcal{D}_k, \forall k\}$ randomly and uploads $\{\widetilde{\mathcal{D}}_k, \forall k\}$ as pilot data to the server. The sampling ratio of pilot data is defined as $\rho_k=\lvert \widetilde{\mathcal{D}}_k \rvert/\lvert \mathcal{D}_k \rvert$.  
    \item[2)] Exam generation: Edge server applies $\texttt{VLM}$ on $\{\widetilde{\mathcal{D}}_k\}$ to obtain pilot memory $\{\widetilde{\mathcal{M}}_k\}$. 
    For each $\widetilde{\mathcal{M}}_k$, we generate an exam $\{\widetilde{\mathcal{Q}}_k,\widetilde{\mathcal{A}}_k^*\}$ with $L_k$ question-answer pairs, by prompting LLM with $\widetilde{\mathcal{M}}_k$ in contexts, where $\widetilde{\mathcal{Q}}_k$ is the question set and $\widetilde{\mathcal{A}}_k^*$ is the correct answer set.    
    \item[3)] Practice test: We test the pre-collection memory $\mathcal{M}_0$ on $\widetilde{\mathcal{Q}}_k$, and compare the answer $\widetilde{\mathcal{A}}_k$ with ground truth $\widetilde{\mathcal{A}}_k^*$. The answer score of exam $\widetilde{\mathcal{Q}}_k$ is $\texttt{GAE}_k$. By iterating over all exams $\{\widetilde{\mathcal{Q}}_k\}$, we obtain the desired output $\{\texttt{GAE}_k\}$.
\end{itemize}

\begin{figure}[!t]
    \centering
    \includegraphics[width=\columnwidth]{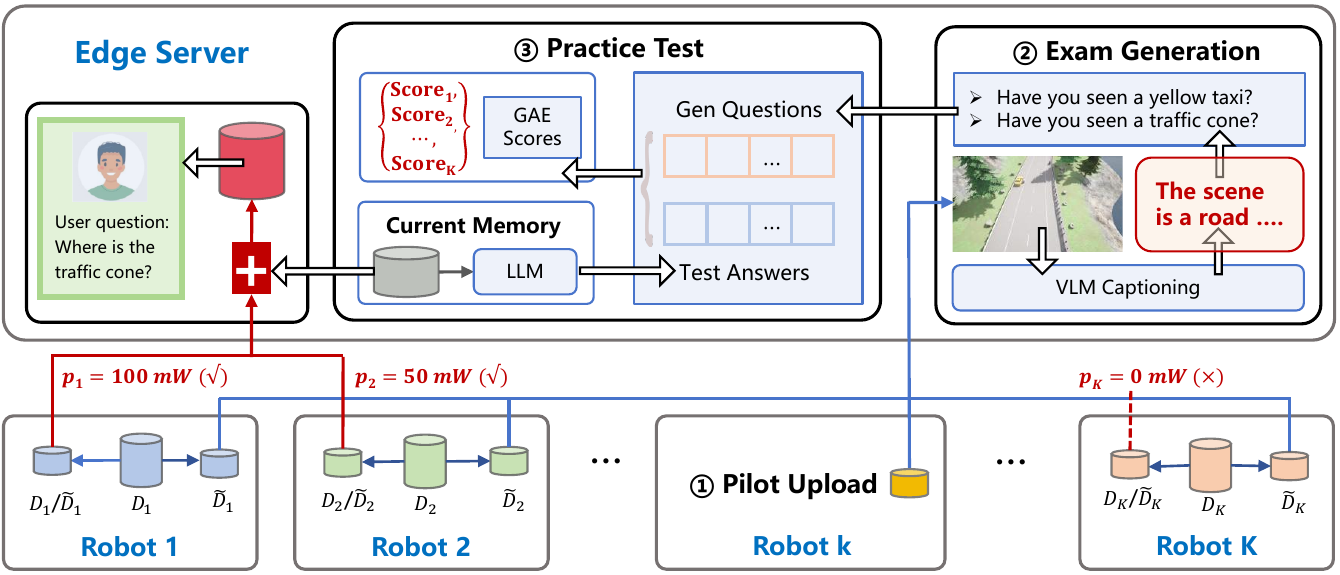}
    \caption{Architecture of GAE for computing QoM.}
    \label{fig:fig2}
    \vspace{-0.1in}
\end{figure}

\subsection{Quality of Memory}

For the above design, if we assume the memory $\mathcal{M}_k$ is perfect on its own exam $\widetilde{\mathcal{Q}}_k$, then the added memory $\mathcal{M}_k$ to $\mathcal{M}_0$ would result in an improvement of $1-\texttt{GAE}_k$ on $\widetilde{\mathcal{Q}}_k$. Therefore, the expected accuracy improvement brought by memory $\mathcal{M}_k$ is $\Delta \Psi_k=\eta_k
\left(1-\texttt{GAE}_k\right)$, where $\eta_k=\frac{|\mathcal{D}_k|}{\sum_{k=1}^K|\mathcal{D}_k|}$ is the percentage of a question coming from dataset $\mathcal{D}_k$. 
With an uploaded portion of 
$F_k(\mathbf{p})/|\mathcal{D}_k|$, the improvement brought by uploaded memory is $\Delta \widetilde{\Psi}_k(\mathbf{p})=\left(1-\texttt{GAE}_k\right)\frac{F_k(\mathbf{p})}{\sum_{k=1}^K|\mathcal{D}_k|}$.

We assume that each uploaded memory chunk contributes independently to the QA capability.
Consequently, the final accuracy is approximated by the sum of $\mathcal{M}_0$-accuracy and the accuracy improvement:
\begin{align}\label{gae}
\Psi\left[\mathcal{M}(\mathbf{p})\right]
\approx
\Psi\left[\mathcal{M}_0\right]
+
\sum_{k=1}^K\left(1-\texttt{GAE}_k\right)\frac{F_k(\mathbf{p})}{\sum_{k=1}^K|\mathcal{D}_k|}, 
\end{align}
where $\Psi\left[\mathcal{M}_0\right]$ is a constant. 
Putting $F_k(\mathbf{p})$ into equation \eqref{gae}, and keeping the terms related to $\mathbf{p}$, 
we obtain the QoM model 
\begin{align}\label{qom}
\texttt{QoM}(\mathbf{p}):=
\sum_{k=1}^K\min\left[
\Delta \Psi_k, {\Theta}_{k}(\mathbf{p})
\right],
\end{align}
where we have defined
\begin{align}
{\Theta}_{k}(\mathbf{p})=\frac{\left(1-\texttt{GAE}_k\right)\widetilde{T}B\mathrm{log}_2\left(1+\frac{H_{k}p_{k}}{\sum_{j\neq k}^KI_{k,j}p_{j}+
\sigma^2} \right)}{Z_{k}\sum_{k=1}^K|\mathcal{D}_k|},
\nonumber
\end{align}
The effective transmission time $\widetilde{T}=T-\Delta T$, where $\Delta T$ is the GAE overhead.
It can be written as $\Delta T=T_{\mathrm{Pilot}}+T_{\mathrm{GAE}}$, 
where $T_{\mathrm{Pilot}}$ is the pilot upload cost, and 
$T_{\mathrm{GAE}}$ is the exam computation cost.

\subsection{Overhead Analysis}

This section will show that $\Delta T$ is much smaller than $T$.
First, since we only need tens of image samples to estimate the memory quality, we can adopt a small sampling ratio with $|\widetilde{\mathcal{D}}_{k}| \ll |{\mathcal{D}}_{k}|$.
By 
\begin{align}
&T_{\mathrm{Pilot}}=\max_{k\in\mathcal{K}}\frac{Z_k|\widetilde{\mathcal{D}}_{k}|}{B
\mathrm{log}_2\left(1+\frac{H_{k}P_{\mathrm{sum}}}{\sum_{j\neq k}^KI_{k,j}P_{\mathrm{sum}}+
\sigma^2K} \right)},
\end{align}
we obtain $T_{\mathrm{Pilot}} \ll T$.
On the other hand, let $T_0$ denote per-QA latency.
For $N_Q=\sum_{k=1}^K L_k$ independent QAs, $N_p$ workers yield $T_{\mathrm{GAE}}\approx\left\lceil N_Q/N_p\right\rceil T_0$. Thus $T_{\mathrm{GAE}}\to T_0$ as $N_p\to N_Q$.
In the following, we will show that $T_{0}\ll T$.

Specifically, if we feed all the memory as a context into the LLM, GAE will scan all information in $\mathcal{M}_0$, which are computationally expensive for transformers. 
However, for a given question, a large history is not required to provide a correct answer.
Instead, only a small subset of the memory ${\mathcal{S}}_0\in\mathcal{M}_0$ is sufficient \cite{anwar2025remembr}.
This converts the one-shot long-horizon inference into a sub-memory search problem, which finds the optimal subset ${\mathcal{S}}^*_0\in\mathcal{M}_0$ to help the LLM generate the correct answer 
    \begin{align}\label{imr}
        \widetilde{\mathcal{A}}_k=\texttt{LLM}(\widetilde{\mathcal{Q}}_k|\mathcal{S}_0^*),
        \ \tilde{\mathcal{S}}_0^*\in\mathcal{M}_0.
    \end{align}
To solve this problem, we adopt iterative memory retrieval in \cite{anwar2025remembr}, where in each iteration we sample a subset $\mathcal{S}'\in\mathcal{M}_0$ and see if question $\mathbf{q}\in\widetilde{\mathcal{Q}}_k$ can be answered by $\mathcal{S}'$. If so, we generate an output formatted as a json file with keys for times, poses, texts. Otherwise, the process repeats by retrieving next-round memory chunk. The chunk is selected by an LLM augmented by the current context. 
The process terminates until the right memory chunk is found, or the maximum number of retrievals is reached. 
According to \cite{anwar2025remembr}, the per-QA latency $T_0$ is only seconds, far below full-memory inference.

\section{Memory Centric Power Allocation}\label{section4}

With QoM model, the original implicit problem $\mathsf{P}$ is converted into the following explicit form
\begin{align}
\mathsf{P}1:~~\mathop{\mathrm{max}}_{\substack{\mathbf{p}\in\mathcal{P}}}
\quad& 
\sum_{k=1}^K\min\left[
\Delta \Psi_k, {\Theta}_{k}(\mathbf{p})
\right].
\end{align}
To solve $\mathsf{P}1$, the only nonconvex parts are functions $\{{\Theta}_{k}(\mathbf{p})\}$.
To tackle these functions, we can leverage the framework of successive optimization, which constructs a sequence of lower bounds $\{\widehat{\Theta}_k\}$ on 
$\{\Theta_k\}$ and replaces $\{\Theta_k\}$ in $\mathsf{P}1$ with $\{\widehat{\Theta}_k\}$ to obtain the surrogate problems. 
Specifically, given any feasible solution $\{\mathbf{p}^\star\}$ to $\mathsf{P}1$, we define surrogate function
\begin{align}
&\widehat{\Theta}_{k}(\mathbf{p}|\mathbf{p}^\star)
=
\frac{\lambda_k}{\mathrm{ln}2}
\Bigg[
\mathrm{ln}\left(\sum_{l=1}^K\frac{I_{k,l}p_{l}}{\sigma^2}+1\right)
\nonumber\\
&
-
\mathrm{ln}\left(\sum_{l=1,l\neq k}^K\frac{I_{k,l}p^\star_{l}}{\sigma^2}+1\right)
\nonumber\\
&
-\left(\sum_{l=1,l\neq k}^K\frac{I_{k,l}p^\star_{l}}{\sigma^2}+1\right)^{-1}\left(\sum_{l=1,l\neq k}^K\frac{I_{k,l}p_{l}}{\sigma^2}+1\right)
+1
\Bigg], 
\nonumber
\end{align}
where 
$\lambda_k=
\frac{\left(1-\texttt{GAE}_k\right)\widetilde{T}B}{Z_{k}\sum_{k=1}^K|\mathcal{D}_k|}$.
Then the following proposition can be established.

\begin{proposition}
The function $\widehat{\Theta}_k$ satisfies the following:

\noindent(i) Lower bound: 
$\widehat{\Theta}_{k}(\mathbf{p}|\mathbf{p}^\star)\leq 
{\Theta}_{k}(\mathbf{p}).$

\noindent(ii) Concavity: $\widehat{\Theta}_{k}(\mathbf{p}|\mathbf{p}^\star)$ is concave in $\mathbf{p}$.

\noindent(iii) Zeroth- and first-order equivalence: 
\begin{align}
    \widehat{\Theta}_{k}(\mathbf{p}^\star|\mathbf{p}^\star) &  = {\Theta}_{k}(\mathbf{p}^\star), 
\
    \nabla_{\mathbf{p}}\widehat{\Theta}_{k}(\mathbf{p}^\star|\mathbf{p}^\star)  = \nabla_{\mathbf{p}} {\Theta}_{k}(\mathbf{p}^\star).
    \nonumber
\end{align}
\end{proposition}
\begin{proof}
Part (i) is proved by checking 
$
\widehat{\Theta}_{k}(\mathbf{p}|\mathbf{p}^\star)- 
{\Theta}_{k}(\mathbf{p})
\leq 0.
$
Part (ii) is proved by verifying the Hessian of $\widehat{\Theta}_k$ is negative semidefinite..
Part (iii) is proved by comparing the function and gradient values of $\{\widehat{\Theta}_k,\Theta_k\}$.
\end{proof}

With part (i) of \textbf{Proposition 1}, a lower bound can be obtained if we replace the functions $\{\Theta_k\}$ by $\{\widehat{\Theta}_k\}$ around a feasible point.
In particular, assuming that the solution at the $n^{\mathrm{th}}$ iteration is given by $\mathbf{p}^{[n]}$, the following problem is considered at the $(n+1)^{\mathrm{th}}$ iteration:
\begin{align}
\mathsf{P}1[n+1]:\mathop{\mathrm{max}}_{\substack{\mathbf{p}\in\mathcal{P}}}
\quad 
\sum_{k=1}^K\min\left[
\Delta \Psi_k, \widehat{\Theta}_{k}(\mathbf{p}|\mathbf{p}^{[n]})
\right].
\label{P2n+1}
\end{align}

Based on part (ii) of \textbf{Proposition 1}, the problem $\mathsf{P}1[n+1]$ can be solved by off-the-shelf software packages (e.g., Mosek) for convex programming.
Denote its optimal solution as
$\mathbf{p}^*$. 
Then we set $\mathbf{p}^{[n+1]}=\mathbf{p}^*$, such that the process repeats with solving the problem $\mathsf{P}1[n+2]$.
According to part (iii) of \textbf{Proposition 1} and \cite[Theorem 1]{sun2016majorization}, every limit point of the sequence
$(\mathbf{p}^{[0]},\mathbf{p}^{[1]},\cdots)$ is a stationary point to $\mathsf{P}1$ as long as the starting point $\mathbf{p}^{[0]}$ is feasible to $\mathsf{P}1$. 
In terms of computational complexity, $\mathsf{P}1[n+1]$ involves $K$ variables.
Therefore, the worst-case complexity for solving $\mathsf{P}1[n+1]$ is
$\mathcal{O}\big(K^{3.5}\big)$.
Consequently, the total complexity for solving $\mathsf{P}1$ is $\mathcal{O}\big(\mathcal{J}\,K^{3.5}\big)$, 
where $\mathcal{J}$ is the number of iterations needed for the algorithm to converge.

To get more insights, we consider a noise-limited operating regime in which the aggregate interference is negligible compared with the receiver noise, i.e., $\sum_{l\ne k}I_{k,l}p_l\ll\sigma^2$.
In this case the problem $\mathsf{P}1$ is approximated as
\begin{align}
\mathsf{P}2:~~
\mathop{\mathrm{max}}_{\substack{\mathbf{p}\in\mathcal{P}}}
\quad&
\sum_{k=1}^K
\min\left[
\Delta \Psi_k, 
\lambda_k
\mathrm{log}_2\left(1+\frac{H_{k}p_{k}}{\sigma^2} \right)
\right]
.
\end{align}
The following proposition gives the optimal solution to $\mathsf{P}2$ (proved based on the Karush-Kuhn-Tucker conditions of $\mathsf{P}2$).
\begin{proposition}
The optimal $\mathbf{p}^*$ to $\mathsf{P}2$ is
\begin{align}\label{pk*}
p_k^{*}(\nu)
= \min\left[\beta_k, \left[\frac{\nu\left(1 - \texttt{GAE}_k\right) \widetilde{T}B}{Z_{k}\sum_{k=1}^K|\mathcal{D}_k|} - \frac{\sigma^2}{H_k}\right]^+\right],
\end{align}
where $\beta_k=
\frac{\sigma^2}{H_k}
  \left(
  2^{\frac{Z_k|D_k|}{\widetilde TB}}-1
  \right)$ is the saturation power, and $\nu$ is the water-filling scalar.
\end{proposition}

The value of $\nu$ in \eqref{pk*} can be obtained as follows.
Let active robot set
$\mathcal{K}_{+}:=\left\{k\in\mathcal{K}:\lambda_k>0\right\}$. 
If $\sum_{k\in\mathcal{K}^+} \beta_k \geq P_{\text{sum}}$, then 
$\nu$ is determined via bisection such that
$\sum_{k\in\mathcal{K}^+}
p_k^{*}(\nu)
= P_{\text{sum}}$. Otherwise, the solution is
  \begin{equation}
      p_k^{*}(\nu)
      =
      \begin{cases}
          \beta_k,
          & k\in\mathcal{K}_{+},\\[2pt]
          0,
          & k\notin\mathcal{K}_{+}.
      \end{cases}
\end{equation}
According to \textbf{Proposition 2}, the transmit power $p_k$ is non-decreasing in $\left(1-\texttt{GAE}_k\right)$, which is the EQA error probability in the simulated exam.
Moreover, if $\texttt{GAE}_k=1$, then we have $p_k^{*}=0$, meaning that this robot memory has no contribution to the MA-EQA.
The above observations disclose that in MA-EQA systems, QoM will have a significant impact on the physical-layer design.

\section{Experiments}\label{section5}

We implement the proposed model and algorithm exploiting Python in CARLA \cite{carla} on a Linux workstation with RTX 3090\,Ti. 
We adopt \textit{Qwen3-VL-8B} for image captioning, \textit{mxbai-embed-large-v1} for text embedding, \textit{Milvus} for database building, and \textit{Qwen3-8B} for EQA inference.
Robots are simulated as drones, which are randomly distributed across the default spawn points of the CARLA Town04 map. Each drone then inspects along a path covering a total distance of 500 meters from its initial spawning location.
The inspection duration is $30$\,s, and the frame rate is $35$\,FPS, resulting in $1050$ recorded frames for each drone. 
Each camera image has a data volume of $Z_k=200$\,KB. 
We randomly generated $10$ objects (i.e., [green car, blue car, white car, purple car, black car, fire truck, yamaha, bus, taxi, traffic cone]).
The questions consist of three types, with a total of $30$ questions in each realization. 
Take fire truck as an example: 
\begin{itemize}
   \item[1)] Is there a fire truck? The answer should be YES/NO; 
   \item[2)] Where is the fire truck? The answer should be an explicit [x, y, yaw] coordinate that is within $50$\,m from the ground truth object location; 
   \item[3)] Which drone sees the fire truck? The answer should be the drone ID $k$.
\end{itemize}

The system bandwidth is $B=10\, \text{MHz}$ and noise power is $\sigma^2=-70$\,dBm \cite{wang2020machine}.
The total time budget is $T=10$\,minutes.
The sum power is $P_{\text{sum}}=200$\,mW.
The robot-server distance is simulated as $d_k\sim\mathcal{U}[50,250]$\,m, where $\mathcal{U}$ denotes uniform distribution.
The wireless channel follows Rayleigh fading, i.e., 
$\mathbf{h}_k = \sqrt{h_0 \omega_k d_k^{-\alpha}}\mathbf{h}^{\text{CN}}_{k}$, where $h_0 = -30 \, \text{dB}$ is the path loss at $1 \, \text{m}$, $\omega_k = -20$\,dB is the fixed shadow fading, $\alpha=3$ is the path loss exponent, and the phase component $\mathbf{h}_{k}^{\mathrm{CN}}\sim\mathcal{CN}(\mathbf{0},\mathbf{I}_N)$ with $N=256$. 
Quantitative results are obtained by averaging $50$ random simulation runs, with independent channel realizations in each run.  

We compare our algorithms to the following baselines:
\begin{itemize}
   \item[1)] \textbf{MaxRate} \cite{ye2025integrated}: Communication sum-rate maximization\footnote{We adopt the objective function and power constraint in \cite{ye2025integrated} and ignore the trajectory and sensing constraints therein.};
   \item[2)] \textbf{MaxCov} \cite{cheng2025development}: Sensing coverage maximization;
   \item[3)] \textbf{Fairness} \cite{zheng2016wireless}: Fairness design by min-rate maximization;
   \item[4)] \textbf{Greedy} \cite{uehara2023k}: Low-GAE-first principle with no consideration for communication conditions;
   \item[5)] \textbf{Remember}\cite{anwar2025remembr}: Baseline standalone QA with $\mathcal{M}=\mathcal{M}_0$ and no multi-agent memory aggregation;
   \item[6)] \textbf{SemCom} \cite{liu2025intelligent}: Adopt semantic similarity of embeddings to select robot memories.
\end{itemize}

\subsection{Experiment 1: Evaluation of GAE}\label{section5-1}

\begin{figure}[!t]
	\centering
    \begin{subfigure}{0.95\linewidth}
		\centering
		\includegraphics[width=\linewidth]{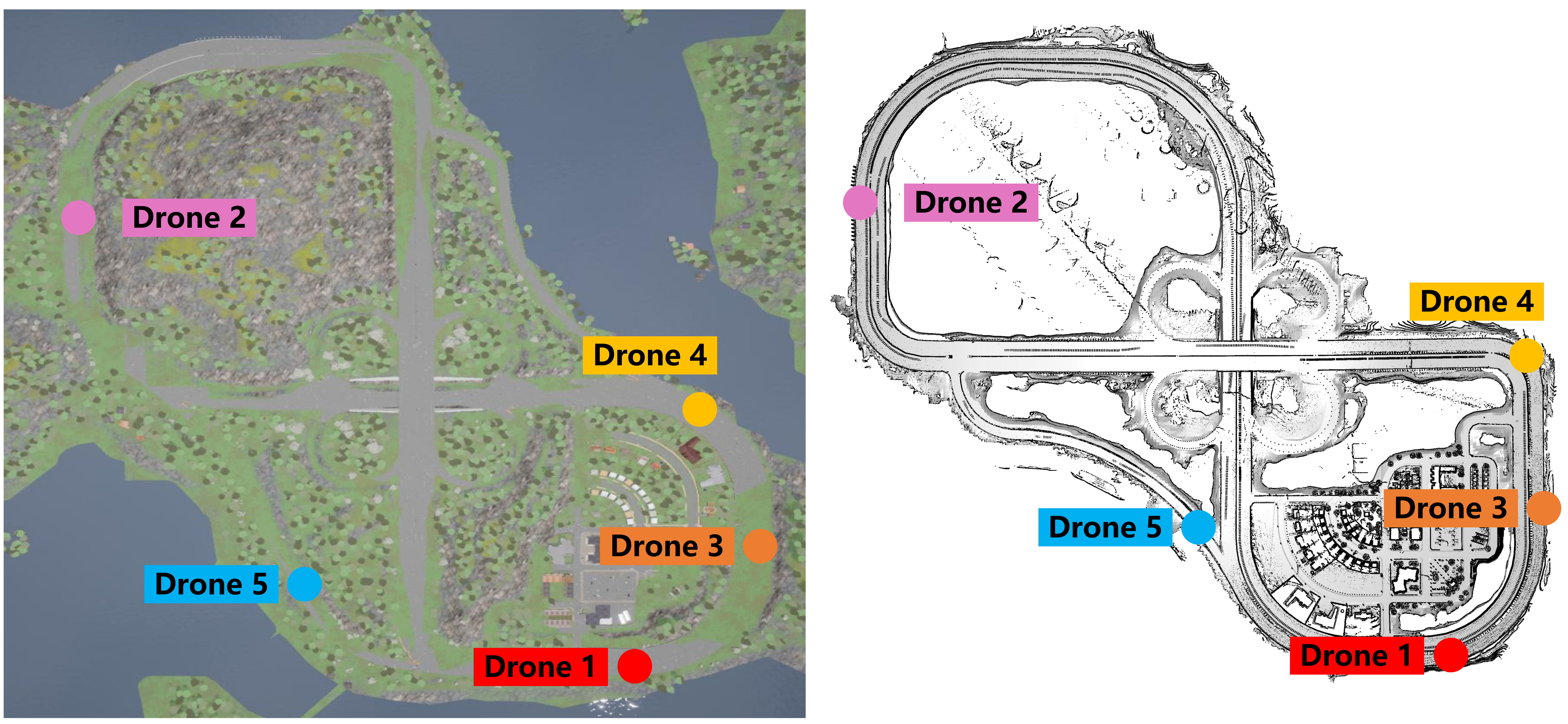}
		\caption{The $5$-robot settings for GAE evaluation.}
	\end{subfigure}
    \\
    \vspace{0.1in}
     	\begin{subfigure}{0.95\linewidth}
		\centering
		\includegraphics[width=\linewidth]{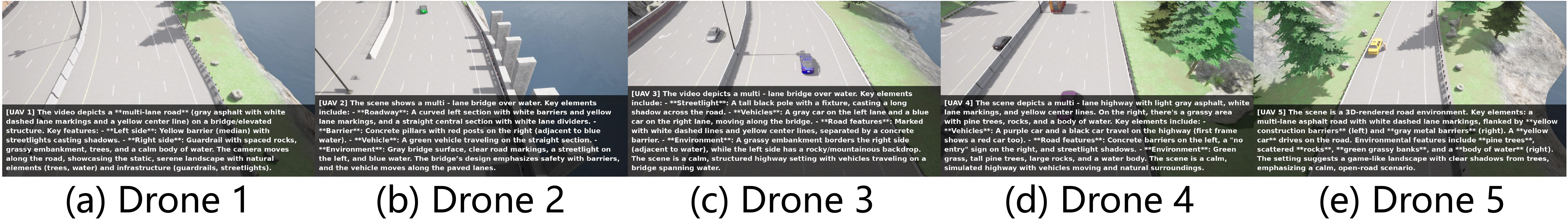}
		\caption{Visualization of robot frames and VLM captions.}
	\end{subfigure}
    \\
    \vspace{0.1in}
	\begin{subfigure}{0.95\linewidth}
		\centering
		\includegraphics[width=\linewidth]{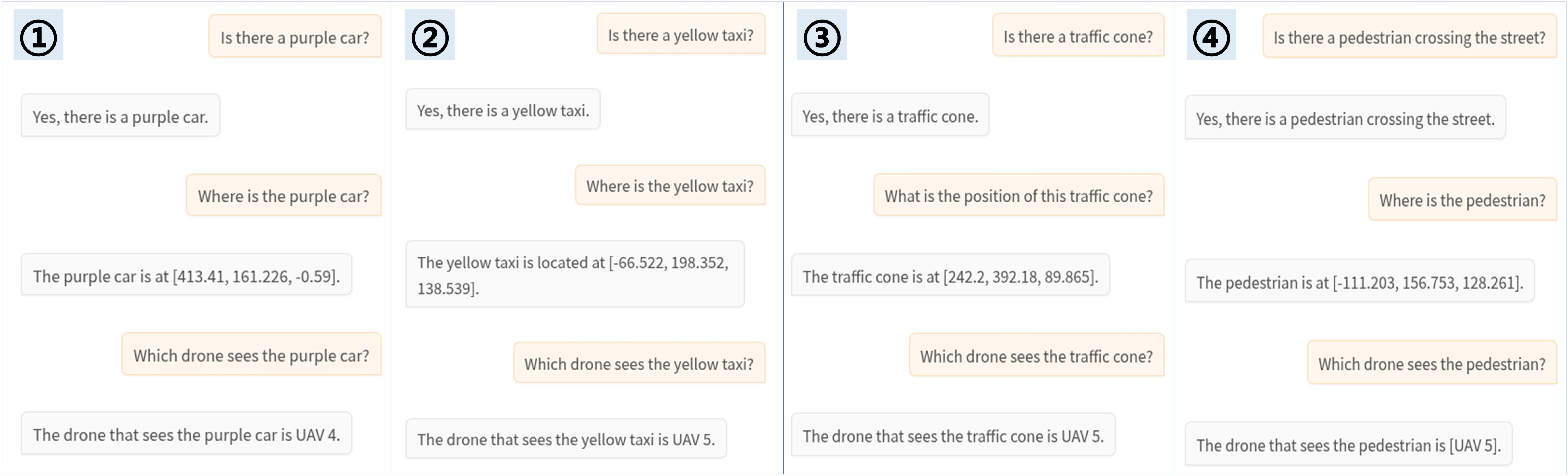}
		\caption{MA-EQA for memory $\mathcal{M}_4\cup\mathcal{M}_0$.}
	\end{subfigure}
    \vspace{-0.02in}
	\caption{Verification of MA-EQA in the $5$-robot scenario.}
	\label{fig:fig3}
    \vspace{-0.2in}
\end{figure}

First, we conduct experiments to validate the effectiveness of the proposed GAE. 
We consider the case of $K=5$ and the spawn points are illustrated in Fig. \ref{fig:fig3}a. 
To simulate the heterogeneous memories, we place $[0,1,2,3,4]$ objects in inspection regions of robots $[1,2,3,4,5]$. 
The VLM captioning results are illustrated in Fig. \ref{fig:fig3}b. 
We observe that the scene understanding is accurate under \textit{Qwen3-VL-8B}. 

Then we set historical memory as $\mathcal{M}_0=\mathcal{M}_5$, and measure the QoMs of robots 1--5. 
The results averaged over $50$ random runs are shown in Table~\ref{TabI}. 
In each run, we generate $L_k=5$ questions from $10$ randomly sampled images for robot $k$.
It can be seen that both the SemCom and GAE models identify the memory similarity between the sampled images from robot $5$ and the initial $\mathcal{M}_0$.
However, the SemCom model fails in distinguishing the memory values of other robots, as their similarity scores are close to each other. 
In contrast, the GAE scores of robots 1--4 match the fact that the semantic information becomes richer as the drone index increases.
This implies that the GAE model can not only understand semantic similarity, but also semantic richness. 

Lastly, we test the final EQA accuracy under different memory aggregation schemes: $
\{\mathcal{M}_0\cup \mathcal{M}_k\}_{k=1}^{5}$.
It can be seen from Table~\ref{TabI} that the novelty score $(1-\texttt{GAE}_k)$ is positively correlated with the EQA accuracy improvement obtained after incorporating $\mathcal{M}_k$.
This shows that maximizing $\texttt{QoM}\left[\mathcal{M}(\mathbf{p})\right]$ in \eqref{qom} indeed helps in enhancing MA-EQA.
The EQA for memory $\mathcal{M}_4\cup\mathcal{M}_0$ is shown in Fig. \ref{fig:fig3}c. The answers regarding object existence, positions, and observing robots of purple car, yellow taxi, traffic cone, pedestrian are all correct.

\begin{table}[!t] 
  \centering
  \caption{Evaluation of GAE. \colorbox{purple2}{Purple} denotes the memory \textbf{selected} by the method. \textbf{Bold} denotes the \textbf{highest} value.}
  \resizebox{0.47\textwidth}{!}{%
  \begin{tabular}{cccc}
    \toprule
    \multirow{2}{*}{Robot Index} 
    & SemCom \cite{liu2025intelligent} 
    & \multicolumn{1}{c}{GAE (ours)} 
    & EQA Accuracy \\
    \cmidrule(lr){2-2} \cmidrule(lr){3-3} \cmidrule(lr){4-4} 
    & BERT Similarity & $\texttt{GAE}_k$ & $\mathcal{M}_0\cup\mathcal{M}_k$ \\
    \midrule
    $k=1$ & 0.89 & \textbf{0.79}  & 39\% \\
    $k=2$ & \cellcolor{purple2}0.86 & 0.45   & 48\% \\
    $k=3$ & 0.87 & 0.35  & 59\% \\
    $k=4$ & 0.90 & \cellcolor{purple2}0.27  & \cellcolor{purple2}\textbf{70\%} \\
    $k=5$ & \textbf{0.94} & 0.72  & 38\% \\
    \bottomrule
  \end{tabular}
  }
  \label{TabI}
\end{table}

\begin{figure*}[!t]
    \centering
    \includegraphics[width=0.95\textwidth]{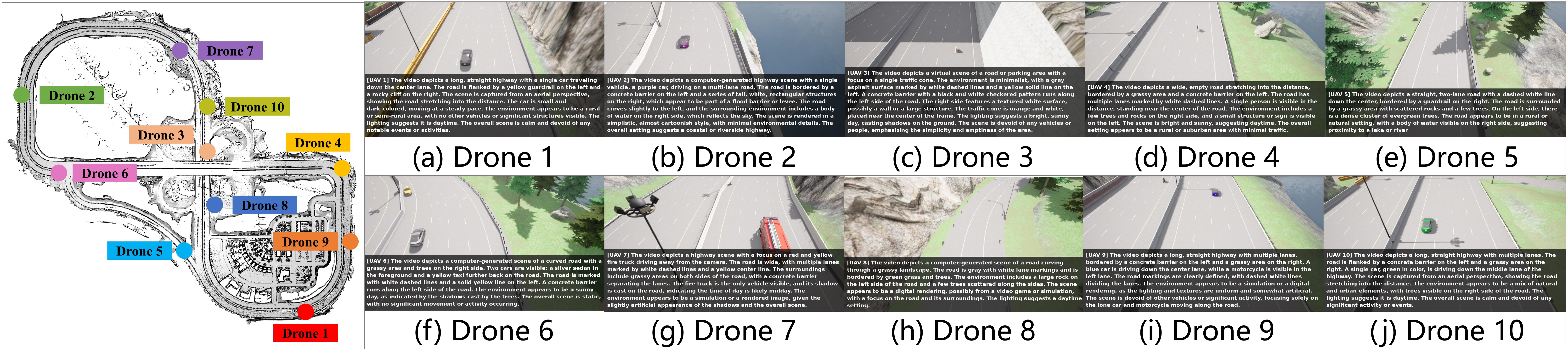}
    \caption{Visualization of drone configurations, image frames, and VLM captions.}
    \label{fig:fig4}
\end{figure*}

\subsection{Experiment 2: Evaluation of MCPA}

Next, we consider the case of $K=10$, where the 10-drone scenario configuration is shown in Fig. \ref{fig:fig4}.
We randomly choose $5$ robots to build the initial ground memory $\mathcal{M}_0$.
For GAE, we generate $L_k=10$ questions from $10$ randomly sampled images for each drone $k$. 
The quantitative results are illustrated in Table~\ref{TabII}. 
Our MCPA achieves the highest EQA accuracy of $95.35\%$ and best QoMs of $0.21$. 
Compared to the standalone memory Remember \cite{anwar2025remembr}, the gain is over $50\%$, which demonstrates the necessity of adopting memory aggregation in city-scale EQA. 
Although MaxRate achieves the highest communication rate, its EQA accuracy is over $30$ points behind our scheme. 
This implies that communication is not always beneficial to memorization, since closer robots may convey useless data that cannot update the knowledge database at the server.
The Greedy scheme may over-allocate resources to remote robots, resulting in fewer connected robots and compromised robustness.
The approach closest to ours is MaxCov, which connects the largest number of drones. 
However, MaxCov also leads to inferior performance, since sensing is not guaranteed to result in new memory.
New memory is created from sensing only when the sensed region contains information that is sufficiently fresh.

\begin{table}[!t]
\centering
\setlength{\tabcolsep}{1.5mm}
\renewcommand{\arraystretch}{1.35}
\caption{Quantitative results of QA tasks. \colorbox{purple2}{Purple} denotes the \textbf{best} performance. 
\#Drones denotes the number of drones with more-than-half data uploads.}
\scalebox{0.96}{
\begin{tabular}{lc|cccc}
\toprule
{Method} & Ref. & EQA Accuracy $\uparrow$ & QoM $\uparrow$ & \#Drones$\uparrow$ & Sum-Rate$\uparrow$ \\
\midrule
MaxCov & \cite{cheng2025development} & 86.09\% & 0.19 & \cellcolor{purple2}\textbf{7.5} & 22.28\,Mbps \\
Greedy & \cite{uehara2023k} & 77.60\% & 0.14 & 3.4 & 30.77\,Mbps \\
Fairness & \cite{zheng2016wireless} & 61.88\% & 0.09 & 1.5 & 10.64\,Mbps \\
MaxRate & \cite{ye2025integrated} & 61.69\% & 0.09 & 3.6 & \cellcolor{purple2}59.04\,Mbps \\
Remember & \cite{anwar2025remembr} & 41.0\% & -- & -- & --  \\
\midrule
MCPA & (\textbf{Ours})  & \cellcolor{purple2}\textbf{95.35\%} & \cellcolor{purple2}\textbf{0.21} & 6.0 & 17.84\,Mbps \\
\bottomrule
\end{tabular}
}
\label{TabII}
\end{table}

\begin{figure*}[!t]
	\centering
	\begin{subfigure}{0.24\linewidth}
		\centering
		\includegraphics[width=1\linewidth]{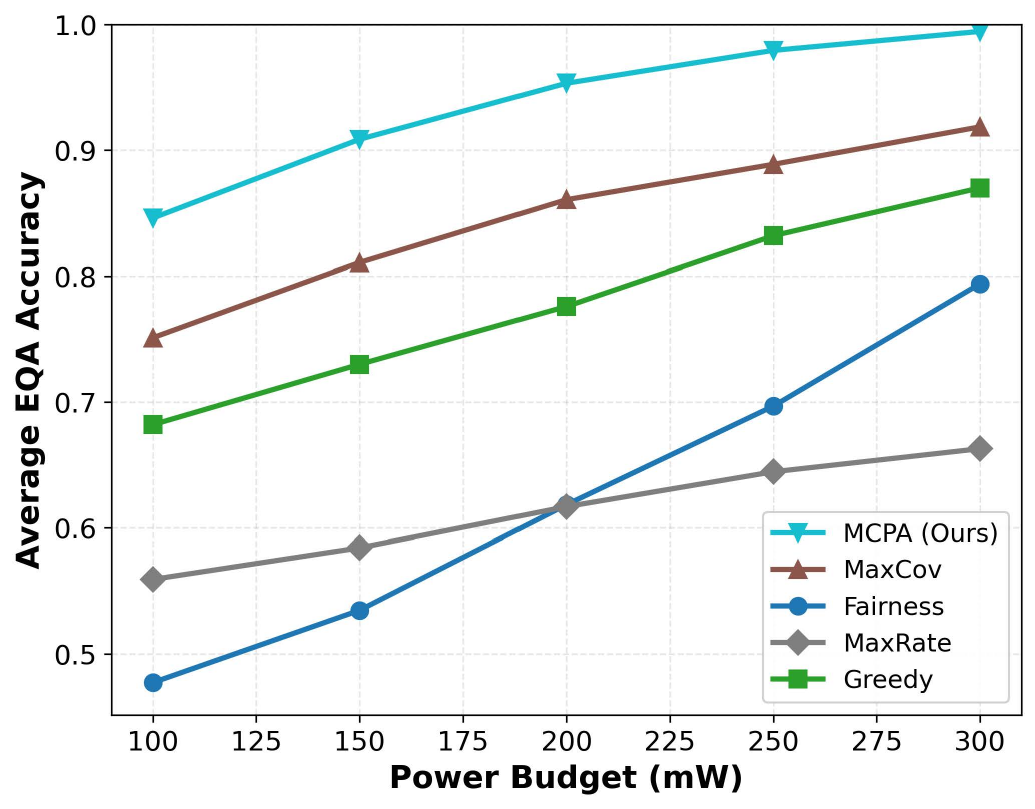}
		\caption{EQA accuracy versus $P_{\mathrm{sum}}$.}
	\end{subfigure}
 	\begin{subfigure}{0.24\linewidth}
		\centering
		\includegraphics[width=1\linewidth]{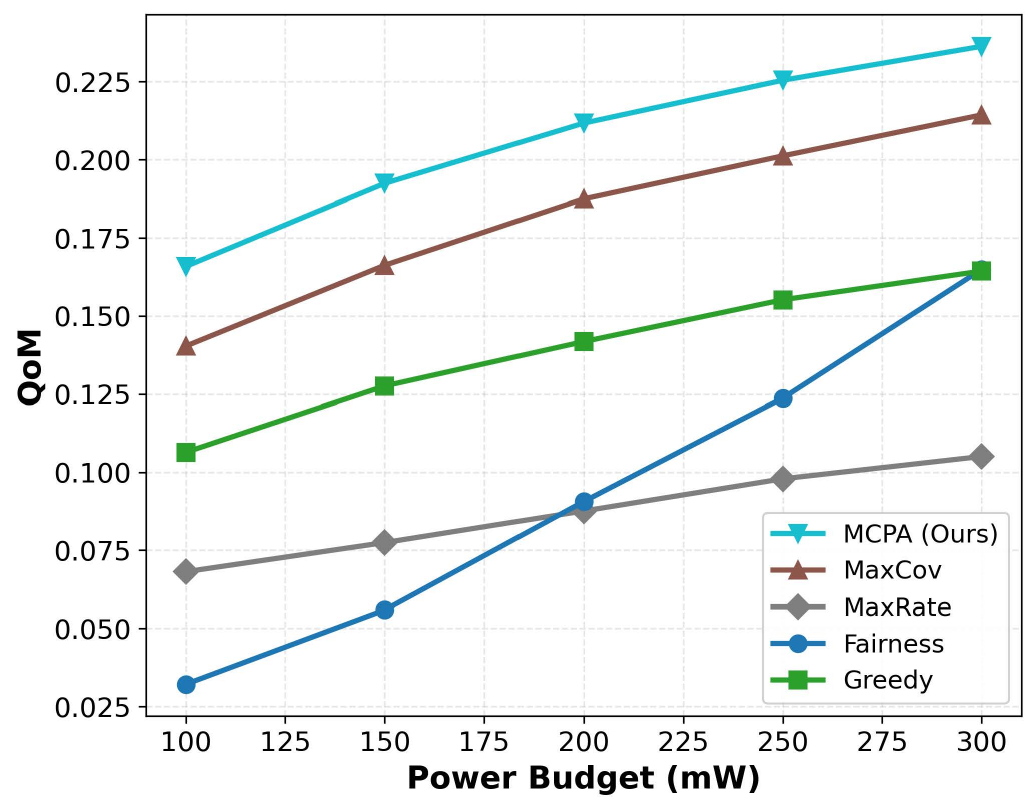}
        \caption{QoM versus $P_{\mathrm{sum}}$.}
	\end{subfigure}
     \begin{subfigure}{0.24\linewidth}
		\centering
		\includegraphics[width=1\linewidth]{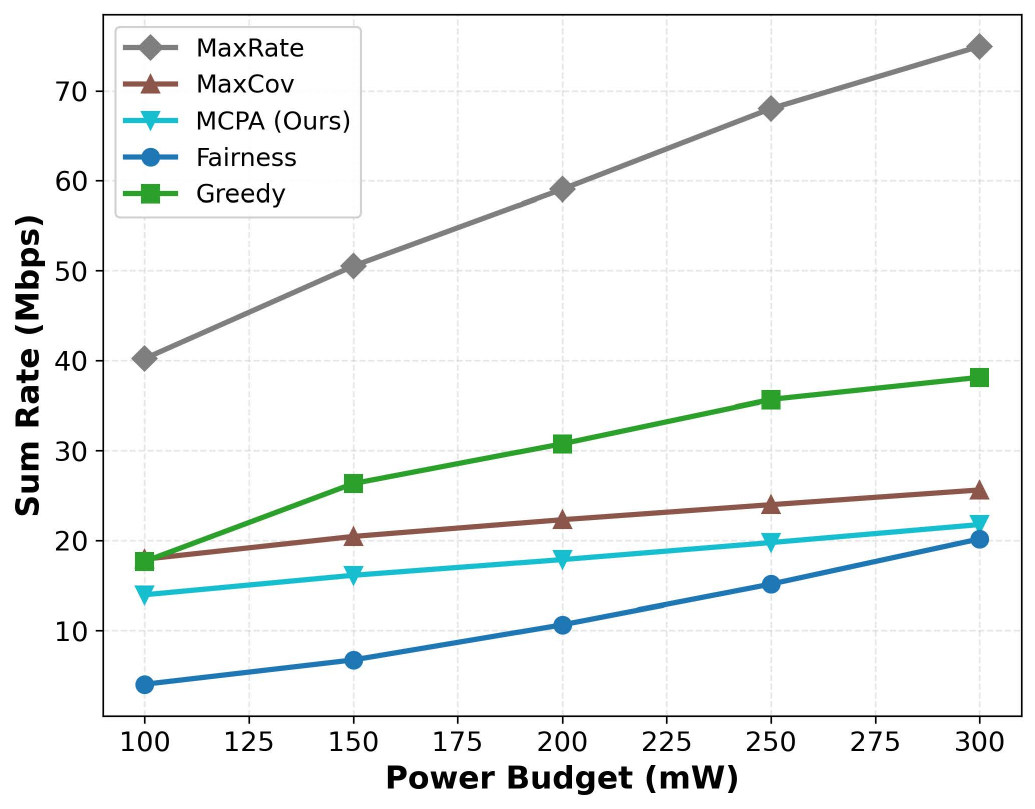}
		\caption{Sum-rate versus $P_{\mathrm{sum}}$.}
	\end{subfigure}
    \begin{subfigure}{0.24\linewidth}
		\centering
		\includegraphics[width=1\linewidth]{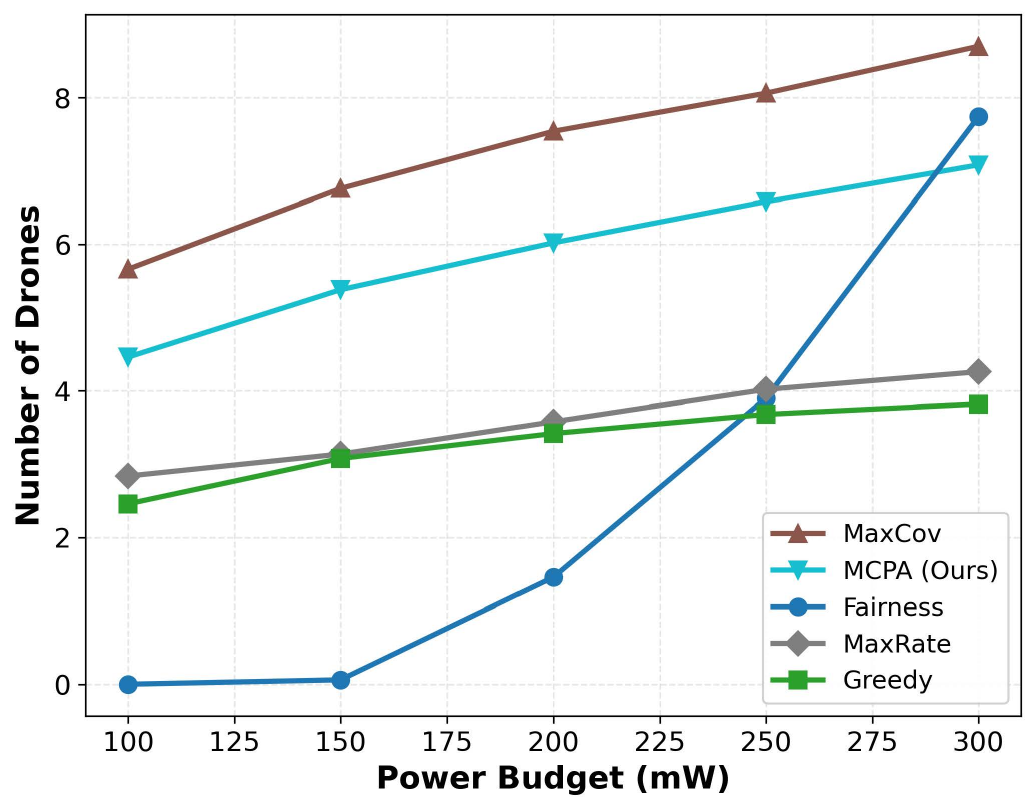}
		\caption{Drone numbers versus $P_{\mathrm{sum}}$.}
	\end{subfigure}
	\caption{Comparison of EQA accuracy, QoM, sum-rate, and robot quantity under different power budgets.}
	\label{fig:fig5}
    \vspace{-0.05in}
\end{figure*}

Finally, we compare the proposed MCPA to the benchmark schemes under different power budgets $P_{\mathrm{sum}}=\{100,150,200,250,300\}$\,mW. 
As shown in Fig.~\ref{fig:fig5}, the EQA accuracies, QoMs, sum-rates, and robot quantities of all schemes increase as the power budget increases. 
But no matter how the resource budget varies, our MCPA demonstrates superior accuracy and QoM compared to four other methods across all the power budgets. 
This confirms its adaptiveness to different configurations.
We also observe from Fig.~\ref{fig:fig5}c that the MaxRate scheme achieves $3.3\times$ data rate of our MCPA, and from Fig.~\ref{fig:fig5}d that the MaxCov scheme attains the largest inspection coverage. However, both schemes underperform our MCPA in terms of memorization. This demonstrates that communication or sensing is not always beneficial to memorization, confirming that not all robot memories matter.

\section{Conclusion}\label{section6}

This paper introduced the memory centric communication paradigm to the emerging MA-EQA application. 
By adopting a GAE pipeline for forward EQA simulation, a QoM model was developed.
Building on QoM, MCPA was derived for memory efficient robotic resource allocation.
The proposed framework establishes a QoM-based interface that bridges embodied intelligence tasks and physical-layer decisions.
Future work will investigate the joint optimization of EQA accuracy and robot battery lifetime, with stochastic optimization for time-varying channels.

\bibliographystyle{IEEEtran}
\bibliography{ref}

\end{document}